\documentclass{article}

\usepackage[preprint]{neurips_2023}

\usepackage[utf8]{inputenc} 
\usepackage[T1]{fontenc}    
\usepackage{hyperref}       
\usepackage{url}            
\usepackage{booktabs}       
\usepackage{amsfonts}       
\usepackage{nicefrac}       
\usepackage{microtype}      
\usepackage{xcolor}         
\usepackage[linesnumbered,ruled]{algorithm2e}
\usepackage{amsmath}
\usepackage{amssymb}
\usepackage{graphicx}
\usepackage{amsthm}
\usepackage{dsfont}
\newtheorem{theorem}{Theorem}[section]

\newtheorem{claim}[theorem]{Claim}

\newtheorem{lemma}[theorem]{Lemma}

\newtheorem{definition}[theorem]{Definition}

\newtheorem{assumption}[theorem]{Assumption}

\usepackage{etoolbox}
\usepackage{subcaption}

\makeatother

\title{Byzantine-Robust Clustered Federated Learning}

\author{%
	Zhixu Tao \\
	ORFE, Princeton University\\
	Princeton, NJ 08544 \\
	\texttt{zhixu.tao@princeton.edu} 
	\And
	Kun Yang \\
	ORFE, Princeton University \\
	Princeton, NJ 08544 \\
	\texttt{ky8517@princeton.edu} 
	\And
	Sanjeev R. Kulkarni \\
	ECE\&ORFE, Princeton University\\
	Princeton, NJ 08544 \\
	\texttt{kulkarni@princeton.edu} 
}

\begin{document}

\maketitle

\begin{abstract}
 This paper focuses on the problem of adversarial attacks from Byzantine machines in a Federated Learning setting where non-Byzantine machines can be partitioned into disjoint clusters. In this setting, non-Byzantine machines in the same cluster have the same underlying data distribution, and different clusters of non-Byzantine machines have different learning tasks. Byzantine machines can adversarially attack any cluster and disturb the training process on clusters they attack. In the presence of Byzantine machines, the goal of our work is to identify cluster membership of non-Byzantine machines and optimize the models learned by each cluster. We adopt the Iterative Federated Clustering Algorithm (IFCA) framework of Ghosh et al. (2020) to alternatively estimate cluster membership and optimize models. In order to make this framework robust against adversarial attacks from Byzantine machines, we use coordinate-wise trimmed mean and coordinate-wise median aggregation methods used by Yin et al. (2018). Specifically, we propose a new Byzantine-Robust Iterative Federated Clustering Algorithm to improve on the results in Ghosh et al. (2019). We prove a convergence rate for this algorithm for strongly convex loss functions. We compare our convergence rate with the convergence rate of an existing algorithm, and we demonstrate the performance of our algorithm on simulated data. \footnote{Our code is adapted from \url{https://github.com/jichan3751/ifca}, and is available at \url{https://github.com/kun0906/brifca}.}
\end{abstract}

\section{Introduction}

 Federated Learning (FL) is a new privacy-preservation machine learning paradigm \cite{rodriguez2023survey, mcmahan2017federated}. Under a federated learning setting, worker machines collaboratively solve a machine learning task under the coordination of a central sever while all data remains decentralized at each worker machine \cite{kairouz2021advances}. The FL paradigm allows us to solve two challenges that traditional machine learning methods suffer from. The first one is the data privacy issue. In the traditional machine learning setting, training a model requires access to all data. However, some data might be personal and sensitive, for example, health data from hospitals. Federated Learning tackles this problem by not allowing the central server to have access to data held by machines. The second challenge is the growing amount of data sets and limited communication and storage capacity. In the traditional machine learning setting, the size of the training data is generally large in order to get a meaningful model. If the central sever uses data from all worker machines to solve a machine learning task, the communication between machines and central sever can be very expensive and inefficient. Moreover, storage at the server might also be an issue. Federated Learning solves this problem by machines limiting the amount of information sent to the central server. 

In Federated Learning, the existence of Byzantine machines makes robustness a huge concern. Byzantine machines can adversarially manipulate the training process and corrupt the training models. It has been shown that even a single Byzantine fault may be able to change the trained model significantly \cite{blanchard2017machine}. Therefore, it is important to design algorithms that are robust against adversarial attacks from Byzantine machines. This problem has been considered in a series of work \cite{blanchard2017machine, blanchard2017byzantine,yin2018byzantine, chen2017distributed}. 

Another concern in Federated Learning is the issue of data heterogeneity. In Federated Learning, worker machines often collect and generate data in a very non-identically distributed way across the network \cite{9084352}. For example, users from different countries use different languages, so in the next-word prediction task, machines are distributed in a highly non-identically distributed way. The data heterogeneity issue has been receiving more attention \cite{kulkarni2020survey, qu2022rethinking}. Some research work focuses on learning one global model from non-i.i.d data \cite{zhao2018federated, mohri2019agnostic, li2019rsa}. An alternative solution is to learn several distinct local models simultaneously \cite{smith2017federated, fallah2020personalized, jiang2019improving, chen2018federated}. In our work, we focus on a specific heterogeneous data setting in which worker machines can be partitioned into disjoint clusters, and each cluster learns a distinct model. 

In this work, we assume a subset of machines is Byzantine. We assume non-Byzantine machines, or normal machines, can be partitioned into $k$ disjoint clusters based on their data distributions. Normal machines in the same cluster draw data points i.i.d from the same distribution. Different clusters correspond to different data distributions and have different learning tasks. We design a robust algorithm to cluster normal machines into $k$ groups and at the same time optimize models for each group. Our algorithm adopts the Iterative Federated Clustering Algorithm (IFCA) framework \cite{ghosh2020efficient} and uses two robust aggregation methods, namely coordinate-wise median and coordinate-wise trimmed mean \cite{yin2018byzantine}. We give assumptions on the fraction of Byzantine machines based on the worst case scenario that all Byzantine machines can collectively attack the smallest cluster. We show that even in the presence of Byzantine machines, our algorithm converges for strongly convex loss functions. Also, we do not have the assumption that all normal machines have the same number of data points. We only impose an assumption on the minimum number of data points across all normal machines. We compare our result with the convergence result of an existing algorithm and show that our algorithm achieves better error rate.  

\section{Related work}
\textbf{Clustering in Federated Learning:} There is a line of work trying to solve non-i.i.d data issues in Federated Learning by clustering machines into different groups. For example, \cite{briggs2020federated} introduced a hierarchical clustering step to separate clusters of machines by the similarity of their local updates to the global joint model. \cite{sattler2020clustered} proposed a post-processing method that is performed after a Federated Learning algorithm converges to group clients based on geometric properties of the Federated Learning loss surface. \cite{li2021federated} proposed an iterative soft clustering algorithm which allows clients to be clustered into more than one cluster. \cite{long2023multi} introduced a multi-center aggregation method to cluster clients, which learns multiple global models as cluster centers and simultaneously computes the optimal matching between clients and centers. Our work adopts the framework from \cite{ghosh2020efficient}. The Iterative Federated Clustering Algorithm (IFCA) proposed by \cite{ghosh2020efficient} solves the problem wherein different groups of clients have their own leaning tasks. More specifically, they assume there are $m$ machines that can be clustered into $k$ disjoint clusters $S_1^*,\dots, S_k^*$. These $k$ clusters correspond to $k$ different underlying distributions $\mathcal{D}_1^*,\dots, \mathcal{D}_k^*$. In other words, if machine $i$ is from cluster $S_j^*$, it draws $n$ i.i.d data points from $\mathcal{D}_j^*$. In our work, we use a similar setting but with Byzantine machines. We formulate our problem in Section \ref{setup}. IFCA alternatively estimates the cluster identities of clients and optimizes model parameters for each cluster.  

\textbf{Robustness in Federated Learning:} In order to tackle adversarial attacks, there are several different directions to enhance the robustness of Federate Learning algorithms. The first method is robust aggregation. The most common way to aggregate information is by using FedAvg \cite{mcmahan2017communication} which takes the weighted average of estimated parameters sent by machines. However, FedAvg is susceptible to outliers and adversarial attacks. \cite{wu2020federated} and \cite{pillutla2019robust} proposed to aggregate by using the geometric mean. \cite{pillutla2022robust} used a multidimensional generalization of the median, known as geometric median or $L_1$ median to do aggregation. \cite{blanchard2017machine} proposed an aggregation method in a distributed learning setting based on combining majority-based and squared-distance-based methods together. In our work, we adopt the aggregation method from \cite{yin2018byzantine}. In \cite{yin2018byzantine}, they proposed two robust aggregation methods, namely coordinate-wise median and coordinate-wise trimmed mean. They showed that these two aggregation methods can achieve order-optimal error rates under some conditions. It is worth of noting that during the preparation of this paper, a recent work by \cite{zhu2023byzantine} proposes several new robust Federated Learning protocols which improve error rates from \cite{yin2018byzantine}. This recent work suggests a future direction which is discussed in Section \ref{limit}. Another method to tackle adversarial attacks is to detect all Byzantine machines and put them into a separate cluster that will not participate in the training process. For example, \cite{sattler2020clustered} assumed a subset of machines is Byzantine and adopted the IFCA framework to cluster Byzantine machines into a separate cluster. \cite{wang2022brief} proposed an algorithm called BRIEF which uses DBSCAN clustering algorithm to detect malicious machines and then implements model segmentation in each cluster to thwart attacks from malicious machines. 

\textbf{Robust Clustering in Federated Learning:} It is worth of mentioning that clustering machines in the presence of Byzantine machines has been getting more attention in Federated Learning. Specifically, \cite{ghosh2019robust} solves the exact same problem as our work. However, their algorithm is different from ours. In \cite{ghosh2019robust}, they designed a Three-Stage algorithm which consists of local empirical risk minimization, estimated parameters clustering and model optimization on each cluster. In other words, they solve the problem by first clustering and then optimizing. Our algorithm is essentially different since we do clustering and optimizing alternatively. We make thorough comparison in section \ref{theory}.

\section{Problem formulation}\label{setup}
In this section, we give a precise formulation of our problem. We consider the problem of minimizing some empirical loss functions in a Federated Learning setting. Suppose there are $m$ worker machines and one center machine. We assume that among these $m$ worker machines, an $\alpha$ fraction of them are Byzantine machines while the remaining $1-\alpha$ fraction of machines are normal. We denote $\mathcal{B}$ as the set of Byzantine machines. For those $(1-\alpha)m$ normal machines, we assume there are $k$ different underlying data distributions $\mathcal{D}_1^*,\dots, \mathcal{D}_k^*$, and these normal machines can be clustered into $k$ disjoint clusters $S_1^*,\dots, S_k^*$. If a normal machine $i$ is from cluster $S_j^*$, it draws $n_i$ i.i.d data $z^{i,1},\dots, z^{i,n_i}$ from distribution $\mathcal{D}_j^*$. We assume that the normal machines do not know their cluster identity. Normal machines communicate with the center machine according to some predefined protocols. Let $f(\theta; z):\Theta \rightarrow \mathbb{R}$ be the loss function associated with a data point $z$. We assume the parameter space $\Theta\subset \mathbb{R}^d$ is convex and compact with diameter $D$. We denote $n_{min}:= \min_{i\in [m]\setminus\mathcal{B}}n_i$ as the minimum number of data points across all normal machines and $N := \sum_{i\in[m]\setminus\mathcal{B}} n_i$ as the total number of data points over all normal machines. Let $F^j(\theta) :=\mathbb{E}_{z\sim \mathcal{D}_j^*}[f(\theta;z)]$ be the population loss function associated with distribution $\mathcal{D}_j^*$. Our goal is to find a set of estimators $\{\theta_j\}_{j=1}^k$ that is close to $\{\theta_j^*\}_{j=1}^k$ where $\theta_j^* = \arg\min_{\theta\in \Theta}F^j(\theta)$. In order to achieve our goal, we minimize the empirical loss function associated with the $i$-th normal machine $F_i(\theta) :=\frac{1}{n_i}\sum_{l=1}^{n_i}f(\theta;z^{i, l})$. Note that throughout the paper, we use $\|\cdot\|$ to denote $\ell_2$ norm. 

One example of our problem formulation is mean estimation with adversarial attacks. Suppose $\alpha m$ machines are Byzantine machines and the remaining machines are normal. If a normal worker machine $i$ is from the $j$-th cluster $S_j^*$, it draws $n_i$ i.i.d data from distribution $\mathcal{D}_j^*$ with mean $\theta_j^*$: $z^{i, 1},\dots, z^{i, n_i}\sim \mathcal{D}_j^*(\theta_j^*)$. The loss function we use is the standard square loss function $f(\theta;z) = \|\theta-z\|^2$. It's easy to verify that $\theta_j^*$ is the minimizer of the population loss function $F^j(\theta) = \mathbb{E}_{z\sim \mathcal{D}_j^*}[\|z-\theta\|^2]$ for the $j$-th cluster. By minimizing the empirical loss function $F_i(\theta) = \frac{1}{n_i}\sum_{l=1}^{n_i}\|z^{i,l}-\theta\|^2$ associated with the $i$-th normal worker machine, we are using the sample mean to estimate the population mean. 

Another example in a supervised learning setting is the case of linear models with square loss \cite{ghosh2020efficient}. We assume for a normal machine $i\in S_j^*$, it has $n_i$ feature-response pairs $z^{i, l} = (x^{i,l}, y^{i,l})$ and each pair satisfies $y^{i,l} = \langle x^{i,l},\theta_j^*\rangle+\epsilon^{i, l}$ where $x^{i, l}\sim\mathcal{N}(0, I_d)$ and $\epsilon^{i, l}\sim \mathcal{N}(0, \sigma^2)$ is additive noise that is independent of $x^{i, l}$. The square loss function is $f(\theta; z) = (y - \langle x,\theta\rangle)^2$. It's clear that $\theta_j^*$ is the minimizer of the population loss function $F^j(\cdot)$.

\section{Algorithm}\label{algorithm}
In this section, we present our algorithm and related definitions. Our algorithm \textit{Byzantine-Robust Iterative Federated Clustering Algorithm} mainly adopts the scheme of \textit{Iterative Federated Clustering Algorithm} (IFCA) \cite{ghosh2020efficient}, alternatively minimizing loss functions and estimating cluster identities. At the same time, our algorithm uses two robust aggregation methods, coordinate-wise median and coordinate-wise trimmed mean, from \cite{yin2018byzantine} to tackle the problem of Byzantine machines. This algorithm starts with $k$ initial estimated parameters $\{\theta_j^{(0)}\}_{j=1}^k$. At the $t$-th iteration, the center machine sends the current estimated parameters $\{\theta_j^{(t)}\}_{j=1}^k$ to all worker machines. If the $i$-th worker machine is a normal machine, it estimates its cluster identity by computing $\hat{j}_i = \arg\min_{j\in [k]}F_i(\theta_j^{(t)})$, and then it computes the gradient $g_i = \nabla F_i(\theta_{\hat{j}_i}^{(t)})$ of the empirical loss function at $\theta_{\hat{j}_i}^{(t)}$. Then the normal worker machine sends its current estimate $\hat{j}_i$ and gradient $g_i$ to the center machine. If the $i$-th machine is a Byzantine machine, it can adversarially pick a $\hat{j}_i\in [k]$ as its current estimate for cluster identity, and sends this $\hat{j}_i$ and an arbitrary $d$ dimensional vector to the center machine. After all worker machines send required information to the center machine, the center machine groups worker machines with the same cluster identity estimates into one cluster. Then the center machine has two options to aggregate gradients within one cluster. The first option is by taking the coordinate-wise median and the second option is by taking the coordinate-wise $\beta$-trimmed mean of the clients' gradients. Then the center machine updates estimated parameters for each cluster by using gradient descent. The Euclidean projection $\Pi_{\Theta}$ guarantees that the updated estimates $\{\theta_j^{(t+1)}\}_{j=1}^k$ still stay in the parameter space $\Theta$. Here, the definition of coordinate-wise median and coordinate-wise trimmed mean we use is the same as \cite{yin2018byzantine}.
\begin{definition}\label{def.1}(Coordinate-wise median)
	For a set of vectors $\{\bold{v}_i\}_{i=1}^m\subset \mathbb{R}^d$, the coordinate-wise median of these vectors is defined as a $d$ dimensional vector $\bold{v}:=\operatorname{med}\{\bold{v}_i:i\in [m]\}$ whose $h$-th coordinate is given by the usual one dimensional median $\bold{v}^h:=\operatorname{med}\{\bold{v}_i^h:i\in [m]\}\,\forall h\in [d]$.
\end{definition}
\begin{definition}\label{def.2}(Coordinate-wise trimmed mean)
	For a set of vectors $\{\bold{v}_i\}_{i=1}^m\subset\mathbb{R}^d$ and $\beta\in [0, \frac{1}{2})$, the coordinate-wise $\beta$-trimmed mean is defined as a $d$ dimensional vector $\bold{v}:=\operatorname{trmean}_{\beta}\{\bold{v}_i:i\in [m]\}$ whose $h$-th coordinate is given by $\bold{v}^h = \frac{1}{(1-2\beta)m}\sum_{v\in U_h}v$ where $U_h\subset \{\bold{v}_i^h\}_{i=1}^m$ is obtained by removing the largest and smallest $\beta$ fraction of $\{\bold{v}_i^h\}_{i=1}^m$.
\end{definition}
\begin{algorithm}
	\SetKwInOut{Input}{Input}
	\SetKwInOut{Output}{Output}
	\SetKwFunction{Initialization}{Robust Initialization}
	
	\Input{number of iterations $T$, number of clusters $k$, step size $\gamma$}
	\Output{a set of estimates $\{\theta_j^{(T)}\}_{j=1}^k$}
	
	Initialize $ \theta_1^{(0)},\dots, \theta_k^{(0)}$ \\
	\For{$t = 0, \dots, T-1$}{center machine broadcasts $\{\theta_j^{(t)}\}_{j=1}^k$ to all worker machines\\
		\For{$i$-th worker machine}{compute \begin{equation}
				\hat{j}_i^{(t)} =
				\begin{cases}
					\arg\min_{j\in [k]}F_i(\theta_j^{(t)}) & \text{if $i$ is a normal machine}\\
					\text{adversarially pick one cluster}&\text{if $i$ is a Byzantine machine}
				\end{cases}       
			\end{equation}
			and \begin{equation}
				g_i^{(t)} =
				\begin{cases}
					\nabla F_i(\theta_{\hat{j}_i}^{(t)}) & \text{if $i$ is a normal machine}\\
					* & \text{if $i$ is a Byzantine machine}
				\end{cases}       
			\end{equation}\\
			send back $\hat{j}_i^{(t)}$ and $g_i^{(t)}$ to the center machine}
		center machine aggregates gradients by:\\
		\textbf{Option I (median)}: $g(\theta_j^{(t)}) = \operatorname{med}\{g_i^{(t)}: \hat{j}_i = j\}\forall j\in [k]$\\
		\textbf{Option II (trimmed mean)}: $g(\theta_j^{(t)}) = \operatorname{trmean}_{\beta}\{ g_i^{(t)}: \hat{j}_i = j	\}\forall j\in [k]$\\
		center machine updates estimated parameters by $\theta_j^{(t+1)} = \Pi_{\Theta}(\theta_j^{(t)} - \gamma g(\theta_j^{(t)}))$
	}
	\Return $\theta_1^{(T)}, \dots, \theta_k^{(T)}$
	\caption{Byzantine-Robust Iterative Federated Clustering Algorithm (Byzantine-Robust IFCA)}
	\label{alg_1}
\end{algorithm}

\section{Theoretical guarantees}
\label{theory}
In this section, we present convergence theorems for the Byzantine-Robust IFCA algorithm presented in Section \ref{algorithm} and summarized in Algorithm \ref{alg_1}. In \cite{ghosh2020efficient}, they used a re-sampling technique to remove inter-dependence between the cluster estimation and the gradient computation. Specifically, if we run $T$ parallel iterations, we partition the $n_i$ data points on the $i$-th machine into $2T$ disjoint subsets so that each subset only contains $n_i' = \frac{n_i}{2T}$ data points. We use the same technique. However, since $n_i'$ and $n_i$ are of the same scale, for simplicity, we will just use $n_i$ to prove theoretical results. We first introduce some definitions and assumptions we need for the convergence theorem. 
\begin{definition}(Strong convexity)
	A differentiable function $f: \mathbb{R}^d\rightarrow\mathbb{R}$ is $\lambda$-strongly convex if $f(y)\ge f(x)+\nabla f(x)^T(y-x)+\frac{\lambda}{2}\|y-x\|^2 \forall x, y\in \mathbb{R}^d$.
\end{definition}
\begin{definition}($L$-smoothness)
	A differentiable function $f: \mathbb{R}^d\rightarrow\mathbb{R}$ is $L$-smooth if $\|\nabla f(x)-\nabla f(y)\|\le L\|x-y\|\forall x, y\in \mathbb{R}^d$.
\end{definition}
\begin{definition}(Lipschitz continuity)
	A function $f:\mathbb{R}^d\rightarrow\mathbb{R}$ is Lipschitz continuous if there exists a constant $K\ge 0$ such that $\|f(x)-f(y)\|\le K\|x-y\|\,\forall x,y\in \mathbb{R}^d$.
\end{definition}
\begin{definition}(Absolute skewness)\cite{yin2018byzantine}
	For a one-dimensional random variable $X$, its absolute skewness is defined as $\gamma(X):=\frac{\mathbb{E}[|X-\mathbb{E}[X]|^3]}{\operatorname{Var}(X)^{3/2}}$. For a $d$-dimensional random vector $\bold{X} = [X_1\, X_2\,\dots\,X_d]^T$, its absolute skewness is defined coordinate-wise by $\gamma(\bold{X}) :=[\gamma(X_1) \,\gamma(X_2)\,\dots\,\gamma(X_d)]^T$.
\end{definition}
First, we impose convexity and smoothness assumptions on the loss functions $f$ and $F^j$. 
\begin{assumption}\label{assumption.1}
	The population loss function $F^j(\cdot)$ is $\lambda_F$-strongly convex and $L_F$-smooth $\forall j\in [k]$. 
\end{assumption}
\begin{assumption}\label{assumption.4}(Smoothness of $f$)
	$\forall j\in[k], \,\forall h\in [d]$, $\partial_hf(\cdot;z)$ is $L_h$-Lipschitz and $f(\cdot; z)$ is $L$-smooth where $z\sim \mathcal{D}_j^*$. 
\end{assumption}
Since we take the coordinate-wise median and the trimmed mean of gradient, we need to define $L_f: = (\sum_{h=1}^d L_h^2)^{\frac{1}{2}}$ as the sum of Lipschitz constants of $\partial_h f$. Since we use loss function values to determine cluster identities of normal machines, we also need the following distributional assumptions on $f(\theta;z)$. 
\begin{assumption}(Bounded variance)\label{assumption.2}
	$\forall j\in [k], \,\forall \theta\in \Theta$, the variance of $f(\theta;z)$ is bounded where $z\sim \mathcal{D}_j^*$, i.e., $\exists\,\eta^2$ such that $\mathbb{E}_{z\sim \mathcal{D}_j^*}[(f(\theta;z) - F^j(\theta))^2]\le \eta^2$.
\end{assumption}
In order to use the IFCA framework, we also need the following notations and technical assumptions from \cite{ghosh2020efficient}. We define $\Delta :=\min_{j\ne j'}\|\theta_j^* - \theta_{j'}^*\|$ as the minimum separation of parameters, define $p_j := \frac{|S_j^*|}{m}$ as the fraction of $j$-th cluster, and define $p:=\min_{j\in [k]}p_j$ as the minimum fraction. We use $x\gtrsim y$ to denote $x\ge Cy$ for some sufficiently large constant $C$. 
\begin{assumption}\label{assumption.5}
	Without loss of generality, we assume $\forall j\in [k], \,\|\theta_j^*\|\le 1$. Also, we assume $\forall j\in [k], \, \|\theta_j^{(0)}-\theta_j^*\|\le\frac{1}{4}\sqrt{\frac{\lambda_F}{L_F}}\Delta$, $n_{min}\gtrsim \frac{k\eta^2}{\lambda_F^2\Delta^4}$, $p\gtrsim\frac{\log(N)}{m}$ and $\Delta \ge \tilde{\mathcal{O}}(\max\{n_{min}^{-1/5}, m^{-1/6}n_{min}^{-1/3}\})$.
\end{assumption}
In other words, we require that the initialization has to be good enough, each normal machine has enough data points, and each cluster has normal machines. The assumption on $\Delta$ ensures that at each iteration $t$, $\|\theta_j^{(t)}-\theta_j^*\|$ is small enough $\forall j\in [k]$.
\subsection{Coordinate-wise median analysis}
For coordinate-wise median analysis, we need the following assumptions on the variance of the gradient and absolute skewness. As noted in \cite{yin2018byzantine}, these two assumptions are satisfied in many learning problems. 
\begin{assumption}(Bounded variance of gradient)\label{assumption.7}
	$\forall j\in [k], \,\forall \theta\in \Theta$, the variance of $\nabla f(\theta;z)$ is bounded where $z\sim \mathcal{D}_j^*$, i.e., $\exists\, \nu^2$ such that $\mathbb{E}_{z\sim\mathcal{D}_j^*}[\|\nabla f(\theta;z) - \nabla F^j(\theta)\|^2]\le \nu^2$.
\end{assumption}
\begin{assumption}\label{assumption.3}(Bounded absolute skewness)
	$\forall j\in [k], \,\forall \theta\in \Theta$, $\exists S>0$ such that $\|\gamma(\nabla f(\theta;z))\|_{\infty}\le S$ where $z\sim \mathcal{D}_j^*$.
\end{assumption}
Next, we present the convergence guarantee for the algorithm with Option I, coordinate-wise median aggregation. 
\begin{theorem}\label{thm.1}
	Consider Option I in Algorithm \ref{alg_1}. Suppose assumptions \ref{assumption.1}, \ref{assumption.4}, \ref{assumption.2}, \ref{assumption.5}, \ref{assumption.7} and \ref{assumption.3}all hold true. Suppose the fraction $\alpha$ of Byzantine machines satisfies 
	\begin{align}
		\frac{4c_1\eta^2}{p\delta\lambda^2_F\Delta^4 n_{min}}+\frac{4\alpha}{p}+\sqrt{\frac{d\log(1+NL_fD)}{m(\frac{1}{4}p-\frac{c_1\eta^2}{\delta\lambda_F^2\Delta^4n_{min}}-\alpha)}}+0.4748\frac{S}{\sqrt{n_{min}}}\le \frac{1}{2}-\epsilon
	\end{align}
	for some constants $\epsilon, \delta >0$ and a specific constant $c_1$. Let the step size $\gamma = \frac{1}{L_F}$. With probability at least $1-\delta-\frac{1}{\operatorname{poly}(N)} - \frac{4d}{(1+\frac{1}{4}pmn_{min}L_fD)^d}$, after $T$ parallel iterations, we have $\forall j\in [k]$,
	\begin{align}
		\|\theta_j^{(T)} - \theta_j^*\|\le (1-\frac{\lambda_F}{\lambda_F+L_F})^T\|\theta_j^{(0)} - \theta_j^*\|+\frac{2}{\lambda_F}\mathcal{O}(C_{\epsilon}\nu(\frac{S}{n_{min}}+\frac{\alpha}{\sqrt{n_{min}}}+\sqrt{\frac{d\log(NL_fD)}{n_{min}mp}})).
	\end{align}
	In particular, $C_{\epsilon} = \sqrt{2\pi}\exp(\frac{1}{2}(\Phi^{-1}(1-\epsilon))^2)$ where $\Phi^{-1}$ is the inverse of the cumulative distribution function of the standard Gaussian distribution. 
\end{theorem}
We prove this theorem in Appendix \ref{app.1}. We denote $\omega = \mathcal{O}(C_{\epsilon}\nu(\frac{S}{n_{min}}+\frac{\alpha}{\sqrt{n_{min}}}+\sqrt{\frac{d\log(NL_fD)}{n_{min}mp}}))$. If we run the algorithm for $T\ge \frac{L_F+\lambda_F}{\lambda_F}\log(\frac{\lambda_F}{2\omega}\max_{j\in [k]}\|\theta_j^{(0)}-\theta_j^*\|)$ parallel iterations, by using the inequality $\log(1-x)\le -x$, we get $\|\theta_j^{(T)} - \theta_j^*\|\le\frac{4\omega}{\lambda_F}$. Here we achieve an error rate $\tilde{\mathcal{O}}(\frac{\alpha}{\sqrt{n_{min}}}+\frac{\sqrt{d}}{\sqrt{pmn_{min}}}+\frac{1}{n_{min}})$. It is worth noting that even though the dependence on the dimension in this error rate is $\sqrt{d}$, the upper bound $\nu$ for $\operatorname{Var}(\nabla f(\theta;z))$ also depends on $d$. This will introduce extra dependence on $d$. We will further discuss this limitation and potential solutions in Section \ref{limit}. The assumption on $\alpha$ is based on the worst case scenario that we allow all Byzantine machines to collectively attack the smallest cluster. 
\subsection{Coordinate-wise trimmed mean analysis}
Next, we present the convergence theorem for the algorithm with Option II, trimmed mean aggregation. We need the following sub-exponential property for the coordinate-wise trimmed mean. Compared with the bounded absolute skewness from assumption \ref{assumption.3} for the coordinate-wise median, the sub-exponential property is stronger since it requires all the moments of the partial derivatives are bounded. 
\begin{assumption}($\sigma$-sub-exponential)\label{assumption.6}
	$\forall h\in [d]$ and $\forall \theta\in \Theta$, $\partial_h f(\theta;z)$ is $\sigma$-sub-exponential.
\end{assumption}
\begin{theorem}\label{thm.2}
	Consider Option II in Algorithm \ref{alg_1}. Suppose assumptions \ref{assumption.1}, \ref{assumption.4},  \ref{assumption.2}, \ref{assumption.5} and \ref{assumption.6} all hold true. Suppose the fraction $\alpha$ of Byzantine machines satisfies 
	\begin{align}\label{ineq.4}
		\frac{4c_1\eta^2}{p\delta\lambda_F^2\Delta^4n_{min}}+\frac{4\alpha}{p}\le\beta \le \frac{1}{2}-\epsilon
	\end{align}
	for some constants $\epsilon, \delta >0$. Choose step size $\gamma = \frac{1}{L_F}$. With probability at least $1-\delta-\frac{1}{\operatorname{poly}(N)}-\frac{4d}{(1+\frac{1}{4}pmn_{min}L_fD)^d}$, after $T$ parallel iterations, we have $\forall j\in [k]$, 
	\begin{align}
		\|\theta_j^{(T)} - \theta_j^*\|\le (1-\frac{\lambda_F}{\lambda_F+L_F})^T\|\theta_j^{(0)} - \theta_j^*\|+\frac{2}{\lambda_F}\mathcal{O}(\frac{\sigma d}{\epsilon}(\frac{\beta}{\sqrt{n_{min}}}+\frac{1}{\sqrt{pmn_{min}}})\sqrt{\log(NL_fD)}).
	\end{align}
\end{theorem}
We prove this theorem in Appendix \ref{app.2}. We denote $\omega' = \mathcal{O}(\frac{\sigma d}{\epsilon}(\frac{\beta}{\sqrt{n_{min}}}+\frac{1}{\sqrt{pmn_{min}}})\sqrt{\log(NL_fD)})$. After running the algorithm for $T\ge \frac{L_F+\lambda_F}{\lambda_F}\log(\frac{\lambda_F}{2\omega'}\max_{j\in [k]}\|\theta_j^{(0)}-\theta_j^*\|)$ parallel iterations, we guarantee $\|\theta_j^{(T) }-\theta_j^*\|\le\frac{4\omega'}{\lambda_F}$. Here we achieve the error rate $\mathcal{O}(\frac{\beta d}{\sqrt{n_{min}}}+\frac{d}{\sqrt{pmn_{min}}})$. If $\frac{4\alpha}{p} = \tilde{c}\beta$ for some constant $\tilde{c}$, then the error rate is $\tilde{\mathcal{O}}(\frac{\alpha d}{p\sqrt{n_{min}}}+\frac{d}{\sqrt{pmn_{min}}})$. Since we assume $n_{min}\gtrsim \frac{k\eta^2}{\lambda_F^2\Delta^4}$, the first term $\frac{4c_1\eta^2}{p\delta\lambda_F^2\Delta^4n_{min}}$ in inequality (\ref{ineq.4}) is negligible, which means inequality (\ref{ineq.4}) requires approximately $\frac{4\alpha}{p}\le \beta$. Again, the assumption on $\alpha$ is based on the worst case scenario that all Byzantine machines can collectively attack the smallest cluster. 
\subsection{Comparison with existing algorithm}\label{sect.comparison}
It is important to compare our result with the result of the Three-Stage Algorithm from \cite{ghosh2019robust}. The Three-Stage Algorithm consists of empirical risk minimization, robust clustering of Empirical Risk Minimizers (ERMs), and robust aggregation. Specifically, we analyze the performance of this algorithm when Trimmed $K$-means Clustering Algorithm is used at the second stage and coordinate-wise trimmed mean is used at the third stage. We re-define notations from \cite{ghosh2019robust}. Suppose all machines have the same number of data points $n$. We define $G_S$ as the maximum fraction of mis-clustered points in a cluster after $S$ clustering iterations at stage II. With high probability, suppose $G_S\le \rho$. Define $\tilde{\alpha}_j = \frac{\rho p_j+\alpha }{p_j+\alpha }$. Then according to Theorem 2 from \cite{ghosh2019robust}, the error rate for the $j$-th estimate is $\|\hat{\theta}_j-\theta_j^*\|\le\tilde{\mathcal{O}}(\frac{\tilde{\alpha}_jd}{\sqrt{n}}+\frac{d}{\sqrt{p_jmn}})$. However, this error rate is comparable with our error rate only when $k = 2$ and the parameters $\{\theta_1^*, \theta_2^*\}$ satisfy $\theta_1^* = -\theta_2^*$. In this case, after $3\log m$ steps in stage II, the maximum mis-clustering rate $\rho$ can be 0. Here, $\tilde{\alpha}_j = \frac{\alpha}{p_j+\alpha}$. Since $p_j$ is typically much larger than $\alpha$, we have $\tilde{\alpha}_j\approx\frac{\alpha}{p_j}\le\frac{\alpha}{p}$. This gives an error rate $\tilde{\mathcal{O}}(\frac{\alpha d}{p\sqrt{n}}+\frac{d}{\sqrt{pmn}})$, which is the same as our error rate. When $k$ is large, the error rate will have an extra dependence on $\frac{d^2}{p^2\sqrt{n}}$. See Appendix \ref{comparison} for detailed analysis. The sub-optimal error rate is due to the sub-optimal performance of Trimmed $K$-means Clustering Algorithm in high dimension. Moreover, in the Three-Stage Algorithm, after the second clustering stage, the clustering result is fixed and clustering mistake will not be corrected after the second stage, while in our algorithm, we re-cluster at every iteration adaptively based on new estimates. 

In \cite{ghosh2019robust}, the convergence result is guaranteed under a stronger assumption that the empirical risk minimizers $\{\hat{\theta}^{(i)}\}_{i=1}^{(1-\alpha)m}$ corresponding to non-Byzantine machines are sampled from a mixture of $k$ $\sigma$-sub-gaussian distributions. One example that violates this assumption is the mean estimation problem for Poisson distributions. 

Consider $k$ Poisson distributions $\operatorname{Pois}(\lambda_j^*)$ with unknown parameters $\lambda_j^*$. If a normal machine $i$ is from cluster $S_j^*$, it draws $n_i$ i.i.d data points $x^{i,1},\dots, x^{i,n_i}$ from $\operatorname{Pois}(\lambda_j^*)$. Suppose we use the standard squared loss function $f(\lambda, x) = \|x-\lambda\|^2$. The empirical risk minimizer is given by the sample mean $\bar{x}_i =\frac{1}{n_i}\sum_{l=1}^{n_i}x^{i,l}$ which follows $\frac{1}{n_i}\operatorname{Pois}(n_i\lambda_j^*)$. It does not have the sub-Gaussian property.
\section{Experiments}
In this section, we present results of our experiments. For the experiments, we assume all machines have the same number of data points $n=100$. We begin with experiments on a mixture of linear models with standard squared loss functions. We first generate ground truth parameters, or linear regression coefficients, $\{\theta_j^*\}_{j=1}^k$ by drawing $\theta_j^*\sim\operatorname{Bernoulli}(0.5)$ coordinate-wise $\forall j\in [k]$. We re-scale all parameters such that $\|\theta_j^*\|_2 = 1\,\forall j\in[k]$. For $(1-\alpha)m$ non-Byzantine machines, we distribute them evenly in each cluster so that each cluster has $\frac{(1-\alpha)m}{k}$ non-Byzantine machines. For the $i$-th non-Byzantine machine in the $j$-th cluster, we generate data points $\{(x^{i,l}, y^{i,l})\}_{l=1}^n$ by $y^{i,l} = \theta_j^{*T}x^{i,l}+\epsilon^{i,l}$ where $x^{i,l}\sim\mathcal{N}(0, I_d)$ and $\epsilon^{i,l}\sim\mathcal{N}(0,\sigma^2)$ are independent. For a Byzantine machine $b$, we still sample its regression coefficient $\theta_b$ by $\theta_b\sim\operatorname{Bernoulli}(0.5)$ coordinate-wise, but we re-scale it so that $\|\theta_b\|_2=3$ to make outliers. Then for the $b$-th Byzantine machine, its data points $\{(x^{b,l}, y^{b,l})\}_{l=1}^n$ are generated by $y^{b,l} = \theta_b^Tx^{b,l}+\epsilon^{b,l}, x^{b,l}\sim\mathcal{N}(0, I_d), \epsilon^{b,l}\sim\mathcal{N}(0,\sigma^2)$.

At the $t$-th iteration, after the center machine sends current estimates $\{\theta_j^{(t)}\}_{j=1}^k$, each non-Byzantine machine uses loss function values to determine which cluster it belongs to, and computes its gradient. For each Byzantine machine, it also uses loss function values to determine which cluster it belongs to, but for gradient computation, it computes its gradient at $3\theta_j^{(t)}$ so that the gradient it sends back to the center machine will become an outlier again. 

We run three different algorithms. The first one is Algorithm \ref{alg_1} with both aggregation options, coordinate-wise median and coordinate wise trimmed mean, for 300 iterations. The second one is the IFCA Algorithm \cite{ghosh2020efficient} (using $\operatorname{FedAvg}$ for aggregation) for 300 iterations. The third one is the Three-Stage Algorithm \cite{ghosh2019robust} with Trimmed $K$-means Clustering Algorithm for 100 iterations and coordinate-wise trimmed mean aggregation for 300 iterations. We define $\operatorname{dist} = \frac{1}{k}\sum_{j=1}^k\|\hat{\theta}_j-\theta_j^*\|$ as our evaluation metric where $\{\hat{\theta}_j\}_{j=1}^k$ are estimated parameters obtained from the algorithms. The three algorithms are run in the following 4 settings: (a) $k=2, m=80$, (b) $k=5, m= 200$, (c) $k=10, m=400$, and (d) $k=15, m=600$. For all settings, the fraction of Byzantine machines is $\alpha = 0.05$ and the fraction of trimmed points in the trimmed-mean aggregation is $\beta = 0.05$. The noise scale is $\sigma^2 = 0.2$. Also, we run each setting for $d = [20, 50, 100, 200, 500]$. For each experiment, we run 50 trails. 

In Figure \ref{figure1.a} and \ref{figure1.b}, we plot the average $\operatorname{dist}$ over 50 trails with respect to the dimension $d$ for $k=5$ and $k=10$. When $d$ is small, for example, $d = 20$ or $d= 50$, the Three-Stage Algorithm performs the best. However, when $d$ becomes larger, the Three-Stage Algorithm performs the worst. This is due to the performance of the Trimmed $K$-means Clustering Algorithm used at the second clustering stage. When the dimension is low, the Trimmed $K$-means Clustering Algorithm can guarantee good clustering result, while in high dimension, the final clustering result is sub-optimal. We also observe that the IFCA with $\operatorname{FedAvg}$ in general performs worse than Algorithm \ref{alg_1}. This is due to the fact that $\operatorname{FedAvg}$ is not robust against Byzantine machines. 
In Figure \ref{figure1.c}, we plot the average $\operatorname{dist}$ over 50 trails with respect to the number of clusters $k = [2, 5, 10, 15]$. In general, when $k$ becomes larger, all algorithms perform worse, which validates the point that the error rate has dependence on the minimum fraction of cluster size $p$. In our experiments, $p = \frac{1}{k}$. In conclusion, these plots validate our comparison result in Section \ref{sect.comparison} that the Three-Stage Algorithm gives an extra dependence on $\frac{d^2}{p^2\sqrt{n}}$. Therefore, when the dimension is high or when the number of clusters is large, Algorithm \ref{alg_1} outperforms the Three-Stage algorithm. Moreover, with existence of Byzantine machines, Algorithm \ref{alg_1} outperforms the IFCA framework. 
\begin{figure}\label{figure1}
	\begin{subfigure}{.35\textwidth}
		\centering
		\includegraphics[scale = 0.28]{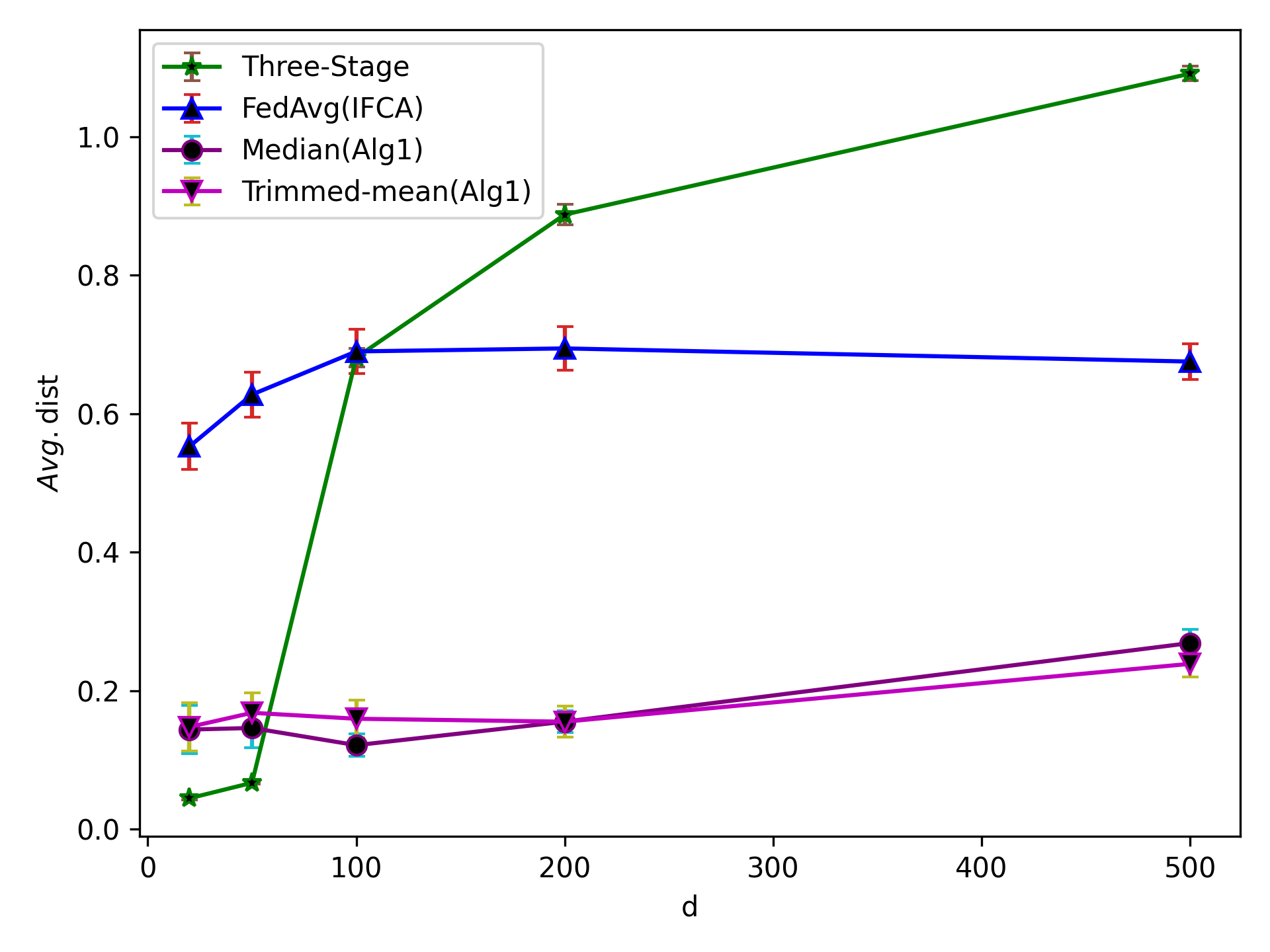}
		\caption{$k=5$}
		\label{figure1.a}
	\end{subfigure}%
\begin{subfigure}{.35\textwidth}
	\centering
	\includegraphics[scale = 0.28]{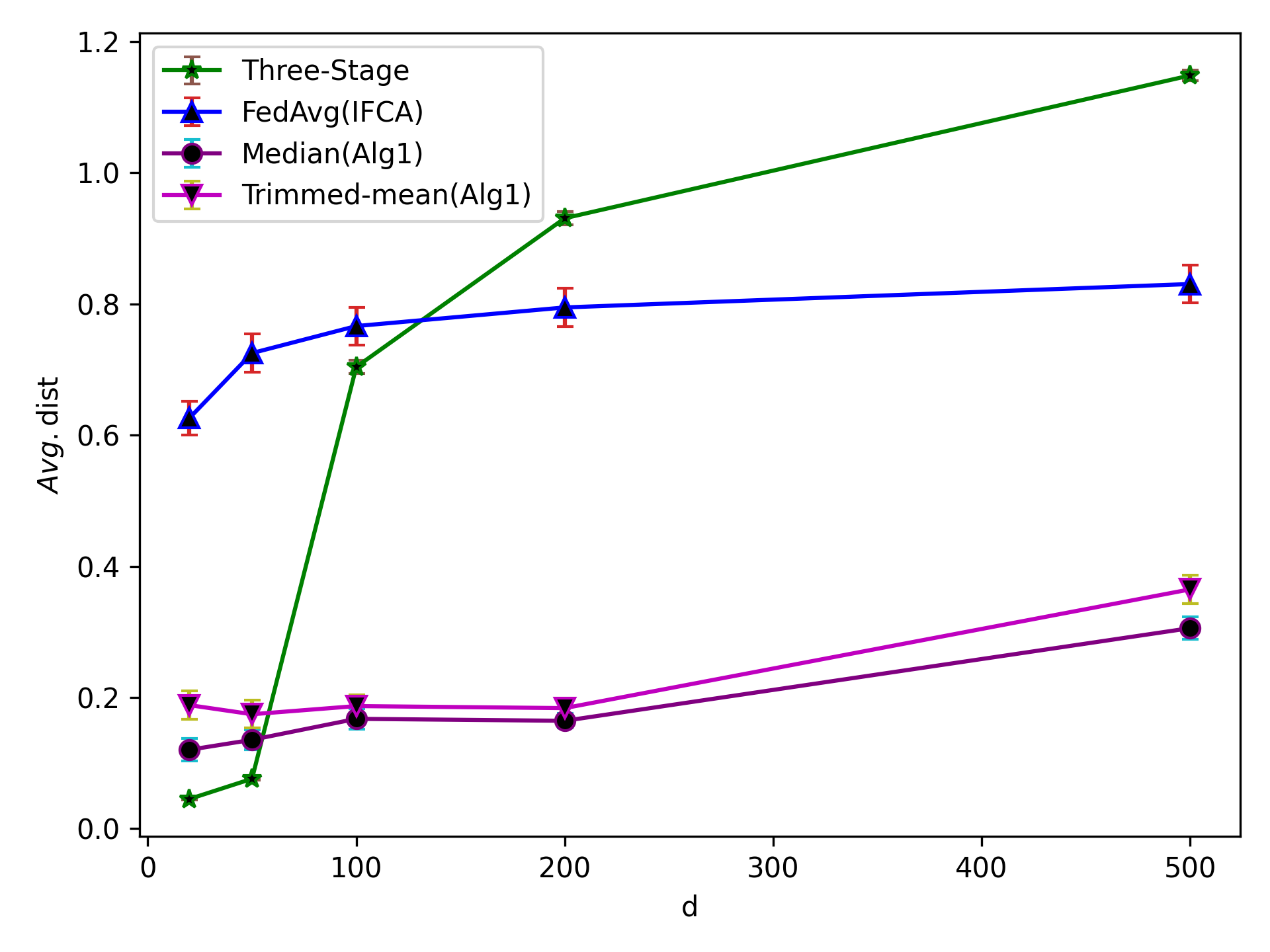}
	\caption{$k=10$}
	\label{figure1.b}
\end{subfigure}%
\begin{subfigure}{.35\textwidth}
	\centering
	\includegraphics[scale = 0.24]{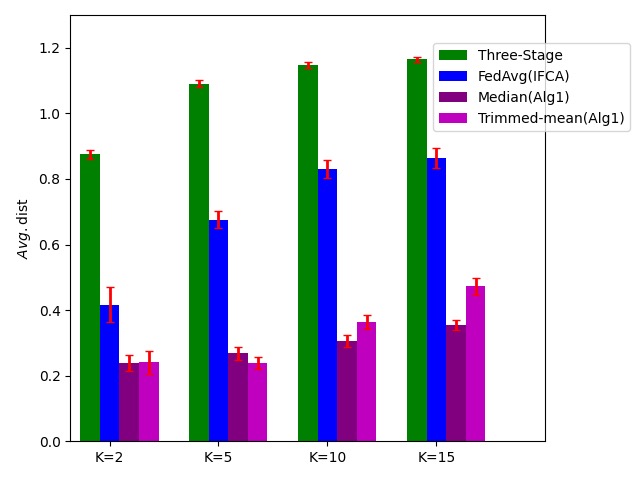}
	\caption{$d=500$}
	\label{figure1.c}
\end{subfigure}
\caption{A comparison of Three-Stage Algorithm, IFCA with $\operatorname{FedAvg}$, Algorithm \ref{alg_1} with Median and Algorithm \ref{alg_1} with Trimmed-mean. In Figure \ref{figure1.a}, $k=5, m=200, n=100,\sigma^2 = 0.2$. In Figure \ref{figure1.b}, $k=10, m= 400, n=100, \sigma^2 = 0.2$. In Figure \ref{figure1.c}, $d = 500$. The error bars in all plots show the standard error over 50 trails. }
\end{figure}

\section{Limitations and future directions}\label{limit}
As mentioned earlier in Section \ref{theory}, the first limitation of our algorithm is that the dependence on dimension $d$ may lead to sub-optimal performance for high dimensional learning problems. The recent paper \cite{zhu2023byzantine}, which studies Byzantine problems in non-clustering setting ($k=1$), proposes several new Byzantine-Robust Federated Learning protocols which improve the dimension dependence and achieve nearly optimal statistical rate with milder assumptions. They also prove a statistical lower bound $\Omega(\sqrt{\frac{\alpha}{n}+\frac{d}{mn}})$. It might be interesting to extend our work and solve the dimension dependence problem by using their new Byzantine-Robust protocols. Moreover, we would like to prove a statistical lower bound for clustering problem with Byzantine machines. A second limitation of our work is the strict requirement of initialization. One approach to address this problem may be to have the machines first send some information to initialize the algorithm. 

\medskip
\bibliographystyle{plainnat}
\bibliography{myref_neurips.bib}
\newpage
\appendix
\textbf{\huge Appendix}
\section{Proof of Theorem \ref{thm.1}}\label{app.1}

Our proof is modified based on proofs from \cite{yin2018byzantine} and \cite{ghosh2020efficient}. The proof of the theorem consists of two parts. First, we prove at the $t$-th step, if the estimate $\theta^{(t)}_j$ is close enough to $\theta_j^*$, the estimate for the $j$-th cluster will also be close to $S_j^*$. Second, we prove conversely when the estimate for the $j$-th cluster is close enough to $S_j^*$ at the $t$-th step, $\theta^{(t+1)}_j$ will be close enough to $\theta_j^*$. We start with one-step analysis.
\subsection{One-step analysis}
Suppose at the $t$-th step, we have $\|\theta_j^{(t)} -\theta_j^*\|\le \frac{1}{4}\sqrt{\frac{\lambda_F}{L_F}}\Delta\,\forall j\in [k]$. We define $S_j^{(t)}:=\{i\in [m]: \hat{j}_i^{(t)} = j\}$ as the set of worker machines clustered into the $j$-th cluster at the $t$-th step. Furthermore, we define $E_i^{j, j'}$ for a normal worker machine $i$ from cluster $S_j^*$ as an event of $i$ being clustered into cluster $S_{j'}^{(t)}$. If $j = j'$, $E_i^{j, j'}$ denotes the event of $i$-th normal worker machine being clustered correctly. Otherwise, it's the event of $i$-th normal worker machine being mis-clustered. 
\begin{lemma}\label{lemma1.1}
	Suppose that a normal worker machine $i\in S_j^*$. Then there exists a constant $c_1$ such that $\forall j'\ne j$, 
	\begin{align}
		\mathbb{P}(E_i^{j, j'})\le c_1\frac{\eta^2}{\lambda_F^2\Delta^4n_{min}}
	\end{align}
	and by union bound, 
	\begin{align}
		\mathbb{P}(\overline{E_i^{j, j}})\le c_1\frac{k\eta^2}{\lambda^2_F\Delta^4n_{min}}.
	\end{align}
\end{lemma} 
\begin{proof}
	The proof of Lemma \ref{lemma1.1} is essentially the same as the proof of Lemma 3 in \cite{ghosh2020efficient}. For completeness, we present the proof here. Without loss of generality, it is enough to bound the probability $\mathbb{P}(E_i^{1, j})\,\forall j\ne 1$. Note that if a normal worker machine $i$ is clustered into cluster $S_j^{(t)}$ for some $j\ne 1$, that means at the $t$-th step, $F_i(\theta_1^{(t)})\ge F_i(\theta_j^{(t)})$. Therefore, $E^{1, j} = \{F_i(\theta_1^{(t)})\ge F_i(\theta_j^{(t)})\}$. Then, $\forall c>0$,
	\[\mathbb{P}(E^{1, j}_i)\le \mathbb{P}(F_i(\theta_1^{(t)})>c) + \mathbb{P}(F_i(\theta_j^{(t)})\le c).\]
	We choose $c = \frac{F^1(\theta_1^{(t)})+F^1(\theta_j^{(t)})}{2}$. Then we have 
	\[\mathbb{P}(F_i(\theta_1^{(t)})>c)  = \mathbb{P}(F_i(\theta_1^{(t)}) - F^1(\theta_1^{(t)})>\frac{F^1(\theta_j^{(t)}) - F^1(\theta_1^{(t)})}{2})\]
	and 
	\[\mathbb{P}(F_i(\theta_j^{(t)})\le c) = \mathbb{P}(F_i(\theta_j^{(t)})-F^1(\theta_j^{(t)})\le -\frac{F^1(\theta_j^{(t)}) - F^1(\theta_1^{(t)})}{2}).\]
	By assumptions, we have $\|\theta_j^{(t)}- \theta_1^*\|\ge \|\theta_j^*-\theta_1^*\| - \|\theta_j^{(t)} - \theta_j^*\|\ge \Delta - \frac{1}{4}\sqrt{\frac{\lambda_F}{L_F}}\Delta\ge \frac{3}{4}\Delta$.\\
	By strong convexity of $F^1$, we have 
	\[F^1(\theta_j^{(t)})\ge F^1(\theta_1^*)+\frac{\lambda_F}{2}\|\theta_j^{(t)} - \theta_1^*\|^2\ge F^1(\theta_1^*)+\frac{9\lambda_F}{32}\Delta^2\]
	and by smoothness of $F^1$, we have 
	\[F^1(\theta_1^{(t)})\le F^1(\theta_1^*)+\frac{L_F}{2}\|\theta_1^{(t)} - \theta_1^*\|^2\le F^1(\theta_1^*) +\frac{\lambda_F}{32}\Delta^2.\]
	Therefore, $F^1(\theta_j^{(t)}) - F^1(\theta_1^{(t)})\ge \frac{\lambda_F}{4}\Delta^2$.\\
	By Chebyshev's inequality, 
	\begin{align*}
		\mathbb{P}(F_i(\theta_1^{(t)})>c)\le\frac{64\eta^2}{\lambda_F^2\Delta^4n_i}\le\frac{64\eta^2}{\lambda_F^2\Delta^4n_{min}}
	\end{align*}
	and similarly, $\mathbb{P}(F_i(\theta_1^{(t)})\le c)\le\frac{64\eta^2}{\lambda_F^2\Delta^4n_{min}}$. Therefore, there exists a universal constant $c_1$ such that $\mathbb{P}(E_i^{j, j'})\le c_1\frac{\eta^2}{\lambda_F^2\Delta^4n_{min}}$.
\end{proof}
\noindent \\
By Lemma \ref{lemma1.1}, it is clear that $\mathbb{E}[|S_j^{(t)}\cap\overline{S_j^*} |]\le c_1\frac{\eta^2 m}{\lambda^2_F \Delta^4 n_{min}}$. By Markov's inequality, with probability at least $1-\delta$, $|S_j^{(t)}\cap\overline{S_j^*} |\le c_1\frac{\eta^2 m }{\delta \lambda_F^2 \Delta^4 n_{min}}$. Also by Lemma \ref{lemma1.1} and the assumption $n_{min}\gtrsim \frac{k\eta^2}{\lambda_F^2\Delta^4}$, we know that $\mathbb{P}(E_i^{j, j})>\frac{1}{2}$ for any normal worker machine $i$. This gives us $\mathbb{E}[|S_j^{(t)}\cap S_j^*|]\ge \frac{1}{2}p_j m\ge\frac{1}{2}pm$. Then by Hoeffding's inequality, 
\[\mathbb{P}(|S_j^{(t)} \cap S_j^*|\le \frac{1}{4}p_jm)\le \mathbb{P}(| |S_j^{(t)} \cap S_j^*|-\mathbb{E}[|S_j^{(t)} \cap S_j^*|]|\ge \frac{1}{4}p_jm)\le 2\exp(-cpm)\] for some constant $c$. Note that $\exp(-cpm)\le \frac{1}{\operatorname{poly}(N)}$. Therefore, with probability at least $1-\frac{1}{\operatorname{poly}(N)}$, $|S_j^{(t)} \cap S_j^*|\ge \frac{1}{4}pm$. \\
\\Next, we state the following claim which is adapted from Claim 2 in \cite{yin2018byzantine}. 
\begin{claim}\label{claim1.1}
	Suppose there are $\tilde{m}$ worker machines and $\tilde{\alpha}\tilde{m}$ machines are Byzantine. Let $\mathcal{B}$ denote the set of Byzantine machines and $\mathcal{M}$ denote the set of normal machines. Suppose each normal machine $i$ has $n_i$ data points. Each normal machine draws data i.i.d from some unknown distribution $\mathcal{D}$. Let $F_i(\theta): = \frac{1}{n_i}\sum_{l=1}^{n_i}f(\theta;z)$ be the empirical loss function and $F:=\mathbb{E}_{z\sim \mathcal{D}}[f(\theta; z)]$ be the population loss function. Denote $N= \sum_{i\in\mathcal{M} }n_i$ and $n_{min} = \min_{i\in \mathcal{M}} n_i$. Define $$g_i(\theta) =
	\begin{cases}
	\nabla F_i(\theta) & \text{if $i$ is a normal machine}\\
	*&\text{if $i$ is a Byzantine machine}
	\end{cases} $$ and $g(\theta) = \operatorname{med}\{g_i(\theta):i\in [\tilde{m}]\}$ as the coordinate-wise median. Suppose all assumptions and settings still hold true and $\tilde{\alpha}$ satisfies 
	\begin{align*}
		\tilde{\alpha} +\sqrt{\frac{d\log(1+NL_fD)}{\tilde{m}(1-\tilde{\alpha})}}+0.4748\frac{S}{\sqrt{n_{min}}}\le \frac{1}{2}-\epsilon
	\end{align*}
	for some $\epsilon>0$. Then with probability at least $1-\frac{4d}{(1+NL_fD)^d}$, we have 
	\begin{align*}
		\|g(\theta) - \nabla F(\theta)\|_2 \le 2\sqrt{2}\frac{1}{\tilde{m}n_{min}}+\frac{\sqrt 2C_{\epsilon}}{\sqrt{n_{min}}}\nu(\tilde{\alpha}+\sqrt{\frac{d\log(1+NL_fD)}{\tilde{m}(1-\tilde{\alpha})}}+0.4748\frac{S}{\sqrt{n_{min}}})\,\forall \theta\in\Theta. 
	\end{align*}
\end{claim}
\noindent We present the proof of this claim later in Appendix \ref{pf.claim1.1}. Intuitively, Claim \ref{claim1.1} tells us when the fraction of bad machines is bounded, the estimate is close enough to the true parameter. In our clustering problem, bad machines in one cluster consist of two parts: Byzantine machines and mis-clustered machines. In other words, define $\alpha_j: = \frac{|S_j^{(t)}\cap \overline{S_j^*}|}{|S_j^{(t)}|}+\frac{|\mathcal{B}\cap S_j^{(t)}|}{|S_j^{(t)}|}$. Then $\alpha_j$ is the fraction of bad machines in the $j$-th estimated cluster. We want to apply Claim \ref{claim1.1} to cluster $S_j^{(t)}$ and $\alpha_j$. With probability at least $1-\frac{1}{\operatorname{poly}(N)}-\delta$,
\begin{align*}
	&\alpha_j+\sqrt{\frac{d\log(1+(\sum_{i\in S_j^{(t)}\cap\mathcal{M}}n_i)L_fD)}{|S_j^{(t)}|(1-\alpha_j)}}+0.4748\frac{S}{\sqrt{\min_{i\in S_j^{(t)}\cap\mathcal{M}}n_i}}\\&\le \frac{|S_j^{(t)}\cap \overline{S_j^*}|+|\mathcal{B}\cap S_j^{(t)}|}{|S_j^{(t)}|}+\sqrt{\frac{d\log(1+NL_fD)}{|S_j^{(t)}|(1-\alpha_j)}}+0.4748\frac{S}{\sqrt{n_{min}}}\\
	&\le \frac{c_1\frac{\eta^2m}{\delta\lambda^2_F\Delta^4n_{min}}+\alpha m}{\frac{1}{4}pm}+\sqrt{\frac{d\log(1+NL_fD)}{|S_j^{(t)}|(1-\alpha_j)}}+0.4748\frac{S}{\sqrt{n_{min}}}.
\end{align*}
Also note that
\begin{align*}
	|S_j^{(t)}|(1-\alpha_j) &= |S_j^{(t)}|-|S_j^{(t)}\cap \overline{S_j^*}|-|\mathcal{B}\cap S_j^{(t)}|\\
	&\ge |S_j^{(t)}\cap S_j^*|-c_1\frac{\eta^2 m}{\delta\lambda_F^2\Delta^4n_{min}}-\alpha m\\
	&\ge \frac{1}{4}pm -c_1\frac{\eta^2 m}{\delta\lambda_F^2\Delta^4n_{min}}-\alpha m.
\end{align*}
Then 
\begin{align*}
	&\alpha_j+\sqrt{\frac{d\log(1+(\sum_{i\in S_j^{(t)}\mathcal{M}}n_i)L_fD)}{|S_j^{(t)}|(1-\alpha_j)}}+0.4748\frac{S}{\sqrt{\min_{i\in S_j^{(t)}\cap\mathcal{M}}n_i}}\\
	&\le \frac{4c_1\eta^2}{\delta\lambda^2_F\Delta^4n_{min}p}+\frac{4\alpha}{p}+\sqrt{\frac{d\log(1+NL_fD)}{m(\frac{1}{4}p - \frac{c_1\eta^2}{\delta\lambda_F^2\Delta^4n_{min}}-\alpha)}}+0.4748\frac{S}{\sqrt {n_{min}}}\\
	&\le \frac{1}{2}-\epsilon\;\text{for some $\epsilon>0$ by our assumption in the theorem.}
\end{align*}
Since $\alpha_j$ and $|S_j^{(t)}|$ satisfy the condition in Claim \ref{claim1.1}, we conclude 
\begin{align}
	&\nonumber\|g(\theta_j^{(t)}) - \nabla F^j(\theta_j^{(t)})\|_2 \le \frac{2\sqrt{2}}{|S_j^{(t)}|\min_{i\in S_j^{(t)}\cap\mathcal{M}}n_i}
	+\frac{\sqrt{2}C_{\epsilon}\nu}{\sqrt{\min_{i\in S_j^{(t)}\cap\mathcal{M}}n_i}}(\alpha_j+\sqrt{\frac{d\log(1+(\sum_{i\in S_j^{(t)}\cap\mathcal{M}}n_i)L_fD)}{|S_j^{(t)}|(1-\alpha_j)}}
	\\\nonumber
	&+\frac{0.4748S}{\sqrt{\min_{i\in S_j^{(t)}\cap\mathcal{M}}n_i}})\\
	&\nonumber\le \frac{2\sqrt{2}}{|S_j^{(t)}|n_{min}}+\frac{\sqrt{2}C_{\epsilon}\nu}{\sqrt{n_{min}}}(\alpha_j+\sqrt{\frac{d\log(1+NL_fD)}{|S_j|(1-\alpha_j)}}+0.4748\frac{S}{\sqrt{n_{min}}})\\
	&\le \frac{2\sqrt{2}}{n_{min}}+\frac{\sqrt{2}C_{\epsilon}\nu}{\sqrt{n_{min}}}(\frac{4c_1\eta^2}{\delta\lambda^2_F\Delta^4n_{min}p}+\frac{4\alpha}{p}+\sqrt{\frac{d\log(1+NL_fD)}{m(\frac{1}{4}p - \frac{c_1\eta^2}{\delta\lambda_F^2\Delta^4n_{min}}-\alpha)}}+0.4748\frac{S}{\sqrt {n_{min}}})\label{ineq.1}
\end{align}
with probability at least $1-\frac{4d}{(1+L_fD\sum_{i\in S_j^{(t)}\cap\mathcal{M}}n_i)^d}\ge 1-\frac{4d}{(1+\frac{1}{4}pmn_{min}L_fD)^d}$ .\\
Next, we prove the convergence result. 
\begin{align*}
	\|\theta_j^{(t+1)}-\theta_j^*\|&=\|\Pi_{\Theta}(\theta_j^{(t)} - \gamma g(\theta_j^{(t)})) - \theta_j^*\| \\
	&\le\|\theta_j^{(t)}-\gamma g(\theta_j^{(t)}) -\theta_j^*\|\;\text{by the property of Euclidean projection}\\
	&=\|\theta_j^{(t)}-\gamma g(\theta_j^{(t)}) +\gamma \nabla F^j(\theta_j^{(t)})-\gamma \nabla F^j(\theta_j^{(t)})-\theta_j^*\|\\
	&\le \|\theta_j^{(t)}-\gamma \nabla F^j(\theta_j^{(t)})-\theta_j^*\|+\gamma \|g(\theta_j^{(t)}) - \nabla F^j(\theta_j^{(t)})\|
\end{align*}
Choose $\gamma = 1/L_F$. Following the proof from \cite{yin2018byzantine}, it is easy to show $\|\theta_j^{(t)}-\gamma \nabla F^j(\theta_j^{(t)})-\theta_j^*\|\le (1-\frac{\lambda_F}{\lambda_F+L_F})\|\theta_j^{(t)} - \theta_j^*\|$. Combining inequality \ref{ineq.1}, we have $\forall j\in [k]$, 
\begin{align}
	\|\theta_j^{(t+1)}-\theta_j^*\|&\nonumber\le (1-\frac{\lambda_F}{\lambda_F+L_F})\|\theta_j^{(t)} - \theta_j^*\|\\
	&+\frac{1}{L_F}\{\frac{2\sqrt{2}}{n_{min}}+\frac{\sqrt{2}C_{\epsilon}\nu}{\sqrt{n_{min}}}(\frac{4c_1\eta^2}{\delta\lambda^2_F\Delta^4n_{min}p}+\frac{4\alpha}{p}+\sqrt{\frac{d\log(1+NL_fD)}{m(\frac{1}{4}p - \frac{c_1\eta^2}{\delta\lambda_F^2\Delta^4n_{min}}-\alpha)}}+0.4748\frac{S}{\sqrt {n_{min}}})\}\label{ineq.2}
\end{align}
By iterating inequality (\ref{ineq.2}), we get the result. 
\subsection{Proof of Claim \ref{claim1.1}}\label{pf.claim1.1}
The proof for Claim \ref{claim1.1} is modified based on the proof of Claim 2 in \cite{yin2018byzantine}. The main modification is for Lemma 1 in \cite{yin2018byzantine}.
\begin{lemma}
	Consider one dimensional random variable robust estimation problem. For the $i$-th normal worker machine, suppose it draws $n_i$ i.i.d samples $\{z^{i, l}\}_{l=1}^{n_i}$ from one dimensional distribution $\mathcal{D}$. Let $\bar{z}_i = \frac{1}{n_i}\sum_{l=1}^{n_i}z^{i, l}$ denote the sample mean on the $i$-th normal machine. Define $\tilde{p}(z): = \frac{1}{\tilde{m}(1-\tilde{\alpha})}\sum_{i\in \mathcal{M}}\mathds{1}(\bar{z}_i\le z)$ where $\mathcal{M}$ is the set of normal machines. Suppose for a fixed $r>0$ and some $\epsilon>0$, we have 
	\[\tilde{\alpha}+\sqrt{\frac{r}{\tilde{m}(1-\tilde{\alpha})}}+0.4748\frac{\gamma(x)}{\sqrt{n_{min}}}\le \frac{1}{2}-\epsilon\;\text{where $x\sim \mathcal{D}$}.\]
	Then with probability at least $1-4e^{-2r}$, we have 
	\[\tilde{p}(\mu+C_{\epsilon}\frac{\sigma}{\sqrt{n_{min}}}(\tilde{\alpha}+\sqrt{\frac{r}{\tilde{m}(1-\tilde{\alpha})}}+0.4748\frac{\gamma(x)}{\sqrt{n_{min}}}))\ge \frac{1}{2}+\tilde{\alpha}\]
	and 
	\[\tilde{p}(\mu - C_{\epsilon}\frac{\sigma}{\sqrt{n_{min}}}(\tilde{\alpha}+\sqrt{\frac{r}{\tilde{m}(1-\tilde{\alpha})}}+0.4748\frac{\gamma(x)}{\sqrt{n_{min}}}))\le \frac{1}{2}-\tilde{\alpha}\]
	where $\sigma^2,\mu$ are variance and mean of $\mathcal{D}$ respectively. 
\end{lemma}
\begin{proof}
	Define $\sigma_i = \frac{\sigma}{n_i}$, $c_i = 0.4748\frac{\mathbb{E}[|x-\mu|^3]}{\sigma^3\sqrt{n_i}}$ and $W_i = \frac{\bar{z}_i-\mu}{\sigma_i}\,\forall i\in [\tilde{m}]$.\\
	Let $\Phi_i$ be the distribution of $W_i\;\forall i \in \mathcal{M}$ and $\tilde{\Phi}_i(z): = \frac{1}{\tilde{m}(1-\tilde{\alpha})}\sum_{i\in \mathcal{M}}\mathds{1}(W_i\le z)$ be the empirical distribution. \\
	By bounded difference inequality, $\forall r>0$, with probability at least $1-2e^{-2r}$, 
	\[|\tilde{\Phi}_i(z) - \Phi_i(z)|\le \sqrt{\frac{r}{\tilde{m}(1-\tilde{\alpha})}}.\]
	Let $z_1\ge z_2$ be constants such that $\Phi_i(z_1)\ge \frac{1}{2}+\tilde{\alpha}+\sqrt{\frac{r}{\tilde{m}(1-\tilde{\alpha})}}$ and $\Phi_i(z_2)\le \frac{1}{2}-\tilde{\alpha}-\sqrt{\frac{r}{\tilde{m}(1-\tilde{\alpha})}}$.\\
	Then by union bound, with probability at least $1-4e^{-2r}$, $\tilde{\Phi}_i(z_1)\ge \frac{1}{2}+\tilde{\alpha}$ and $\tilde{\Phi}_i(z_2)\le \frac{1}{2}-\tilde{\alpha}$.\\
	Next, we want to choose proper $z_1$ and $z_2$. By Berry-Esseen theorem, $\Phi_i(z_1)\ge \Phi(z_1)-0.4748\frac{\mathbb{E}[|x-\mu|^3]}{\sigma^2\sqrt{n_{min}}}$ where $\Phi$ is the cumulative distribution function for standard Gaussian distribution. Therefore, it is enough to find $z_1$ such that 
	\[\Phi(z_1) = \frac{1}{2}+\tilde{\alpha}+\sqrt{\frac{r}{\tilde{m}(1-\tilde{\alpha})}}+0.4748\frac{\mathbb{E}[|x-\mu|^3]}{\sigma^3\sqrt{n_{min}}}.\]
	By Mean Value Theorem, $\exists \zeta\in [0, z_1]$ such that $$\tilde{\alpha}+\sqrt{\frac{r}{\tilde{m}(1-\tilde{\alpha})}}+0.4748\frac{\mathbb{E}[|x-\mu|^3]}{\sigma^3\sqrt{n_{min}}} = z_1\Phi'(\zeta) \ge \frac{z_1}{\sqrt{2\pi}}e^{-\frac{z_1^2}{2}}.$$ 
	Based on our assumption that for some $\epsilon\in (0, \frac{1}{2})$, $\tilde{\alpha}+\sqrt{\frac{r}{\tilde{m}(1-\tilde{\alpha})}}+0.4748\frac{\gamma(x)}{\sqrt{n_{min}}}\le \frac{1}{2}-\epsilon$, we know that $z_1\le \Phi^{-1}(1-\epsilon)$. \\
	Therefore, $ \tilde{\alpha}+\sqrt{\frac{r}{\tilde{m}(1-\tilde{\alpha})}}+0.4748\frac{\gamma(x)}{\sqrt{n_{min}}}\ge \frac{z_1}{\sqrt{2\pi}}\exp(-\frac{1}{2}(\Phi^{-1}(1-\epsilon))^2)$. \\
	Denote $C_\epsilon = \sqrt{2\pi}\exp(\frac{1}{2}(\Phi^{-1}(1-\epsilon))^2)$. Then $z_1\le C_{\epsilon}(\tilde{\alpha}+\sqrt{\frac{r}{\tilde{m}(1-\tilde{\alpha})}}+0.4748\frac{\gamma(x)}{\sqrt{n_{min}}})$.\\
	Similarly, $z_2\ge -C_{\epsilon}(\tilde{\alpha}+\sqrt{\frac{r}{\tilde{m}(1-\tilde{\alpha})}}+0.4748\frac{\gamma(x)}{\sqrt{n_{min}}})$.\\
	Then with probability at least $1-4e^{-2r}$, 
	\[\tilde{\Phi}_i(C_{\epsilon}(\tilde{\alpha}+\sqrt{\frac{r}{\tilde{m}(1-\tilde{\alpha})}}+0.4748\frac{\gamma(x)}{\sqrt{n_{min}}}))\ge \frac{1}{2}+\tilde{\alpha}\]
	and 
	\[\tilde{\Phi}_i(-C_{\epsilon}(\tilde{\alpha}+\sqrt{\frac{r}{\tilde{m}(1-\tilde{\alpha})}}+0.4748\frac{\gamma(x)}{\sqrt{n_{min}}}))\le \frac{1}{2}-\tilde{\alpha}.\]
	\\
	Note that
	\begin{align*}
		\tilde{\Phi}_i(z):=\frac{1}{\tilde{m}(1-\tilde{\alpha})}\sum_{i\in  \mathcal{M}}\mathds{1}(\bar{z}_i\le \sigma_iz+\mu) = \frac{1}{\tilde{m}(1-\tilde{\alpha})}\sum_{i\in  \mathcal{M}}\mathds{1}(\bar{z}_i\le\frac{\sigma}{\sqrt{n_i}}z+\mu).
	\end{align*}
	Also $\forall n_i$, $\frac{\sigma}{\sqrt{n_{max}}}z+\mu\le \frac{\sigma}{\sqrt{n_i}}z+\mu\le \frac{\sigma}{\sqrt{n_{min}}}z+\mu$. Then we have $\tilde{p}(\frac{\sigma}{\sqrt{n_{max}}}z+\mu)\le \tilde{\Phi}_i(z)\le \tilde{p}(\frac{\sigma}{\sqrt{n_{min}}}z+\mu)$. \\
	Therefore, with probability at least $1-4e^{-2r}$,
	\[\tilde{p}(\mu+C_{\epsilon}\frac{\sigma}{\sqrt{n_{min}}}(\tilde{\alpha}+\sqrt{\frac{r}{\tilde{m}(1-\tilde{\alpha})}}+0.4748\frac{\gamma(x)}{\sqrt{n_{min}}}))\ge \frac{1}{2}+\tilde{\alpha}\]
	and 
	\[\tilde{p}(\mu - C_{\epsilon}\frac{\sigma}{\sqrt{n_{max}}}(\tilde{\alpha}+\sqrt{\frac{r}{\tilde{m}(1-\tilde{\alpha})}}+0.4748\frac{\gamma(x)}{\sqrt{n_{min}}}))\le\frac{1}{2}-\tilde{\alpha}.\]
	Since $\frac{-1}{\sqrt{n_{min}}}\le \frac{-1}{\sqrt{n_{max}}}$, we further have 
	\[\tilde{p}(\mu - C_{\epsilon}\frac{\sigma}{\sqrt{n_{min}}}(\tilde{\alpha}+\sqrt{\frac{r}{\tilde{m}(1-\tilde{\alpha})}}+0.4748\frac{\gamma(x)}{\sqrt{n_{min}}}))\le\frac{1}{2}-\tilde{\alpha}.\]
\end{proof}
The rest of the proof follows exactly the same as proof in \cite{yin2018byzantine} except that we substitute $n$ with $n_{min}$. 
\section{Proof of Theorem \ref{thm.2}}\label{app.2}
The proof of Theorem \ref{thm.2} essentially has the same structure as proof for Theorem \ref{thm.1}. The first part is still to prove at the $t$-th step, if the estimate $\theta^{(t)}_j$ is close to $\theta_j^*$, the estimate for the $j$-th cluster is also close to $S_j^*$. We assume at the $t$-th step, we have $ \|\theta_j^{(t)}-\theta_j^*\|\le \frac{1}{4}\sqrt{\frac{\lambda_F}{L_F}}\Delta$. Again, from Lemma \ref{lemma1.1}, we know that with probability at least $1-\delta$, $|S_j^{(t)}\cap \overline{S_j^*}|\le c_1\frac{\eta^2 m}{\lambda_F^2\Delta^4n_{min}}$ and with probability at least $1-\frac{1}{\operatorname{poly}(N)}$, $|S_j^{(t)}\cap S_j^*|\ge \frac{1}{4}pm$. 
\subsection{One-step Analysis}
In order to do one-step analysis, we first present the following claim which is adapted from Claim 5 in \cite{yin2018byzantine}. 
\begin{claim}\label{claim2.1}
	Suppose there are $\tilde{m}$ worker machines and $\tilde{\alpha}\tilde{m}$ machines are Byzantine. Let $\mathcal{B}$ denote the set of Byzantine machines and $\mathcal{M}$ denote the set of normal machines. Suppose each normal worker machine $i$ has $n_i$ data points, and draws data i.i.d from some unknown distribution $\mathcal{D}$. Let $F_i(\theta): = \frac{1}{n_i}\sum_{l=1}^{n_i}f(\theta;z)$ be the empirical loss function and $F:=\mathbb{E}_{z\sim \mathcal{D}}[f(\theta; z)]$ be the population loss function. Denote $N= \sum_{i\in\mathcal{M}}n_i$ and $n_{min} = \min_{i\in \mathcal{M}} n_i$. Define $$g_i(\theta) =
	\begin{cases}
	\nabla F_i(\theta) & \text{if $i$ is a normal machine}\\
	*&\text{if $i$ is a Byzantine machine}
	\end{cases} $$ and $g(\theta) = \operatorname{trmean}_{\beta}\{g_i(\theta):i\in [\tilde{m}]\}$ as the coordinate-wise $\beta$-trimmed mean. Suppose all assumptions and settings still hold true and $\tilde{\alpha}$ satisfies $\tilde{\alpha}\le\beta\le \frac{1}{2}-\epsilon$ for some $\epsilon >0$. Then with probability at least $1-\frac{4d}{(1+NL_fD)^d}$,
	\begin{align*}
		\|g(\theta)-\nabla F(\theta)\|\le \frac{\sigma d}{\epsilon}(\frac{3\sqrt{2}\beta}{\sqrt{n_{min}}}+\frac{2}{\sqrt{\tilde{m}n_{min}}})\sqrt{\log(1+NL_fD)+\frac{1}{d}\log\tilde{m}}+\tilde{\mathcal{O}}(\frac{1}{m\sqrt{n_{min}}}+\frac{\beta}{n_{min}}+\frac{1}{N}).
	\end{align*}
	
\end{claim}
\noindent The proof of Claim \ref{claim2.1} is in Appendix \ref{pf.claim.2.1}. Similarly as the proof for Theorem \ref{thm.1}, we define $\alpha_j = \frac{|S_j^{(t)}\cap \overline{S_j^*}|}{|S_j^{(t)}|}+\frac{|\mathcal{B}\cap S_j^{(t)}|}{|S_j^{(t)}|}$ where $\mathcal{B}$ denotes the set of Byzantine machines. We apply Claim \ref{claim2.1} to cluster $S_j^{(t)}$ and $\alpha_j$. With probability at least $1-\frac{1}{\operatorname{poly}(N)}-\delta$, 
\[\alpha_j\le \frac{4c_1\eta^2}{\delta \lambda_F^2\Delta^4n_{min}p}+\frac{4\alpha}{p}.\]
Then by our assumption, $\alpha_j\le\beta\le \frac{1}{2}-\epsilon$. Applying Claim \ref{claim2.1}, we have 
\begin{align*}
	\|g(\theta_j^{(t)}) - \nabla F^j(\theta_j^{(t)})\|&\le \frac{\sigma d}{\epsilon}(\frac{3\sqrt{2}\beta}{\sqrt{\min_{i\in S_j^{(t)}\cap\mathcal{M}}n_i}}+\frac{2}{\sqrt{|S_j^{(t)}|\min_{i\in S_j^{(t)}\cap\mathcal{M}}n_i}})\sqrt{\log(1+NL_fD)+\frac{1}{d}\log m}\\&+\tilde{\mathcal{O}}(\frac{1}{|S_j^{(t)}|\sqrt{\min_{i\in S_j^{(t)}\cap\mathcal{M}}n_i}}+\frac{\beta}{\min_{i\in S_j^{(t)}\cap\mathcal{M}}n_i}+\frac{1}{N})\\
	&\le  \frac{\sigma d}{\epsilon}(\frac{3\sqrt{2}\beta}{\sqrt{n_{min}}}+\frac{2}{\sqrt{|S_j^{(t)}|n_{min}}})\sqrt{\log(1+NL_fD)+\frac{1}{d}\log m}\\&+\tilde{\mathcal{O}}(\frac{1}{|S_j^{(t)}|\sqrt{n_{min}}}+\frac{\beta}{n_{min}}+\frac{1}{N}).
\end{align*}
Notice that $|S_j^{(t)}|\ge |S_j^{(t)}\cap S_j^*|\ge \frac{1}{4}pm$. Therefore, with probability at least
\[1-\frac{4d}{(1+L_fD\sum_{i\in S_j^{(t)}}n_i)^d}\ge 1-\frac{4d}{(1+\frac{1}{4}pmn_{min}L_fD)^d},\]
we have
\begin{align*}
	\|g(\theta_j^{(t)}) - \nabla F^j(\theta_j^{(t)})\|&\le \frac{\sigma d}{\epsilon}(\frac{3\sqrt{2}\beta}{\sqrt{n_{min}}}+\frac{4}{\sqrt{pmn_{min}}})\sqrt{\log(1+NL_fD)+\frac{1}{d}\log m}\\&+\tilde{\mathcal{O}}(\frac{1}{pm\sqrt{n_{min}}}+\frac{\beta}{n_{min}}+\frac{1}{N}).
\end{align*}
Next, it follows the exactly same argument as the proof for Theorem \ref{thm.1} to get 
\begin{align}\label{ineq.3}
	\|\theta_j^{(t+1)}-\theta_j^*\|\le (1-\frac{\lambda_F}{\lambda_F+L_F})\|\theta_j^{(t)}-\theta_j^*\|+\frac{2}{\lambda_F}\mathcal{O}(\frac{\sigma d}{\epsilon}(\frac{\beta}{\sqrt{n_{min}}}+\frac{1}{\sqrt{pmn_{min}}})\sqrt{\log(NL_fD)})
\end{align}
where in this inequality we omit universal constants and higher order terms. By iterating inequality (\ref{ineq.3}), we get the final result. 
\subsection{Proof of Claim \ref{claim2.1}}\label{pf.claim.2.1}
The proof for Claim \ref{claim2.1} is modified based on the proof of Claim 5 from \cite{yin2018byzantine}. First recall Bernstein's inequality for independent sub-exponential random variables. 
\begin{claim}(Bernstein's inequality)
	Let $X_1, \dots, X_N$ be independent, mean zero and $\sigma_i$-sub-exponential random variables. Then $\forall t\ge 0$, 
	\[\mathbb{P}(|\frac{1}{N}\sum_{i=1}^NX_i|\ge t)\le 2\exp(-N\min\{\frac{t^2}{2(\max_{i}\sigma_i)^2}, \frac{t}{2\max_{i}\sigma_i}\}).\]
\end{claim}
We use Bernstein's inequality to prove the following lemma.
\begin{lemma}
	Consider one dimensional random variable robust estimation problem. For the $i$-th normal worker machine, suppose it draws $n_i$ i.i.d samples $\{z^{i,l}\}_{l=1}^{n_i}$ of some one dimensional random variable $z\sim \mathcal{D}$. Suppose $\mathbb{E}[z] = \mu$ and $z$ is $\sigma$-sub-exponential. Denote $\bar{z}_i = \frac{1}{n_i}\sum_{l=1}^{n_i}z^{i, l}$ as the sample mean on the $i$-th normal machine. Let $\mathcal{B}$ denote the set of Byzantine machines and $n_{min} = \min_{i\in[\tilde{m}]\setminus \mathcal{B}}n_i$. Then $\forall t\ge 0$, 
	\[\mathbb{P}\{|\frac{1}{(1-\tilde{\alpha})\tilde{m}}\sum_{i\in [\tilde{m}]\setminus\mathcal{B}}\bar{z}_i-\mu|\ge t\}\le 2\exp(-(1-\tilde{\alpha})\tilde{m}\min\{\frac{t^2n_{min}}{2\sigma^2}, \frac{t\sqrt{n_{min}}}{2\sigma}\}),\]
	and $\forall s\ge 0$, 
	\[\mathbb{P}\{\max_{i\in [\tilde{m}]\setminus\mathcal{B}}\{|\bar{z}_i-\mu|\}\ge s\}\le 2(1-\tilde{\alpha})\tilde{m}\exp(-n_{min}\min\{\frac{s}{2\sigma}, \frac{s^2}{2\sigma^2}\}).\]
	Furthermore, when $\tilde{\alpha}\le\beta$, $|\frac{1}{(1-\tilde{\alpha})\tilde{m}}\sum_{i\in [\tilde{m}]\setminus\mathcal{B}}\bar{z}_i-\mu|\le t$ and $\max_{i\in [\tilde{m}]\setminus\mathcal{B}}\{|\bar{z}_i-\mu|\}\le s $, we have
	\[|\operatorname{trmean}_{\beta}\{\bar{z}_i:i\in [\tilde{m}]\}-\mu|\le\frac{t+3\beta s}{1-2\beta}.\]
\end{lemma}
\begin{proof}
	By Bernstein's inequality, we know $\forall s\ge 0$ and $i\in [\tilde{m}]\setminus\mathcal{B}$, 
	\[\mathbb{P}\{|\bar{z}_i - \mu|\ge s\}\le 2\exp(-n_i\min\{\frac{s}{2\sigma}, \frac{s^2}{2\sigma^2}\})\le 2\exp(-n_{min}\min\{\frac{s}{2\sigma}, \frac{s^2}{2\sigma^2}\}).\]
	Then by union bound, 
	\[\mathbb{P}\{\max_{i\in [\tilde{m}]\setminus\mathcal{B}}\{|\bar{z}_i-\mu|\}\ge s\}\le 2(1-\tilde{\alpha})\tilde{m}\exp(-n_{min}\min\{\frac{s}{2\sigma}, \frac{s^2}{2\sigma^2}\}).\]
	Next, since $z^{i,l}$ is $\sigma$-sub-exponential, then $\bar{z}_i$ is $\frac{\sigma}{\sqrt{n_i}}$-sub-exponential with mean $\mu$. Also notice that all $\bar{z}^i$'s are independent. Since $\max_{i\in [\tilde{m}]\setminus \mathcal{B}}\frac{\sigma}{\sqrt{n_i}} = \frac{\sigma}{\sqrt{n_{min}}}$, by Bernstein's inequality, we have 
	\[\mathbb{P}\{|\frac{1}{(1-\tilde{\alpha})\tilde{m}}\sum_{i\in [\tilde{m}]\setminus\mathcal{B}}\bar{z}_i-\mu|\ge t\}\le 2\exp(-(1-\tilde{\alpha})\tilde{m}\min\{\frac{t^2n_{min}}{2\sigma^2}, \frac{t\sqrt{n_{min}}}{2\sigma}\}).\]
	Now suppose $\tilde{\alpha}\le\beta$, $|\frac{1}{(1-\tilde{\alpha})\tilde{m}}\sum_{i\in [\tilde{m}]\setminus\mathcal{B}}\bar{z}_i-\mu|\le t$ and $\max_{i\in [\tilde{m}]\setminus\mathcal{B}}\{|\bar{z}_i-\mu|\}\le s $. Let $\mathcal{U}$ denote the set of untrimmed elements and $\mathcal{T}$ denote the set of trimmed elements. Let $\mathcal{M} = [\tilde{m}]\setminus\mathcal{B}$ be the set of normal machines. 
	\begin{align*}
		|\operatorname{trmean}_{\beta}\{\bar{z}_i:i\in [\tilde{m}]\}-\mu|&=|\frac{1}{(1-2\beta)\tilde{m}}\sum_{i\in \mathcal{U}}\bar{z}_i-\mu|\\
		&=\frac{1}{(1-2\beta)\tilde{m}}|\sum_{i\in \mathcal{B}\cap\mathcal{U}}(\bar{z}_i-\mu)+\sum_{i\in\mathcal{M}}(\bar{z}_i-\mu)-\sum_{i\in\mathcal{M}\cap\mathcal{T}}(\bar{z}_i-\mu)|\\
		&=\frac{1}{(1-2\beta)\tilde{m}}(|\sum_{i\in \mathcal{B}\cap\mathcal{U}}(\bar{z}_i-\mu)|+|\sum_{i\in\mathcal{M}}(\bar{z}_i-\mu)|+|\sum_{i\in\mathcal{M}\cap\mathcal{T}}(\bar{z}_i-\mu)|)\\
		&=\frac{1}{(1-2\beta)\tilde{m}}(\beta\tilde{m}\max_{i\in\mathcal{M}}|\bar{z}_i-\mu|+(1-\tilde{\alpha})t+2\beta\tilde{m}\max_{i\in\mathcal{M}}|\bar{z}_i-\mu|)\\
		&\le \frac{t+3\beta s}{1-2\beta}
	\end{align*}
\end{proof}
We apply this lemma to $\partial_hf(\theta;z)$. Recall that we assume $\forall h\in [d]$ and $ \forall \theta\in \Theta$, $\partial_hf(\theta;z)$ is $\sigma$-sub-exponential. Then by above lemma, we have $\forall t\ge 0$ and $s\ge 0$, 
\[\mathbb{P}\{|\frac{1}{(1-\tilde{\alpha})\tilde{m}}\sum_{i\in [\tilde{m}]\setminus\mathcal{B}}g_i^h(\theta)-\partial_h F(\theta)|\ge t\}\le 2\exp(-(1-\tilde{\alpha})\tilde{m}\min\{\frac{t^2n_{min}}{2\sigma^2}, \frac{t\sqrt{n_{min}}}{2\sigma}\})\]
and 
\[\mathbb{P}\{\max_{i\in [\tilde{m}]\setminus\mathcal{B}}\{|g_i^h(\theta)-\partial_hF(\theta)|\}\ge s\}\le 2(1-\tilde{\alpha})\tilde{m}\exp(-n_{min}\min\{\frac{s}{2\sigma}, \frac{s^2}{2\sigma^2}\})\]
where $g_i^h(\theta)$ denotes the $h$-th coordinate of $g_i(\theta)$. Therefore, with probability at least 
\[1-2\exp(-(1-\tilde{\alpha})\tilde{m}\min\{\frac{t^2n_{min}}{2\sigma^2}, \frac{t\sqrt{n_{min}}}{2\sigma}\})- 2(1-\tilde{\alpha})\tilde{m}\exp(-n_{min}\min\{\frac{s}{2\sigma}, \frac{s^2}{2\sigma^2}\}),\]
\[|\operatorname{trmean}_{\beta}\{g_i^h(\theta): i\in [\tilde{m}]\}-\partial_hF(\theta)|\le\frac{t+3\beta s}{1-2\beta}.\]
Next, we choose 
\[t = \frac{\sigma}{\sqrt{n_{min}}}\max\{\frac{8d}{\tilde{m}}\log(1+NL_fD), \sqrt{\frac{8d}{\tilde{m}}\log(1+NL_fD)}\}\]
and 
\[s = \sigma \max\{\frac{4}{n_{min}}(d\log(1+NL_fD)+\log \tilde{m}), \sqrt{\frac{4}{n_{min}}(d\log(1+NL_fD)+\log \tilde{m})}\}.\]
The rest of the proof follows exactly the same as proof of claim 5 from \cite{yin2018byzantine}.
\section{Comparison of Error Rates}\label{comparison}
In this section, we show that the error rate in Theorem 2 \cite{ghosh2019robust} is worse than our error rate $\tilde{\mathcal{O}}(\frac{\alpha d}{p\sqrt{n_{min}}}+\frac{d}{\sqrt{pmn_{min}}})$. We use the same notations as \cite{ghosh2019robust}. Define the minimum fraction of cluster size as $\gamma_1 = \min_{j\in [k]}\frac{|S_j^*|}{(1-\alpha)m}$. Then by our notation, $\gamma_1 = \frac{p}{1-\alpha}$. Define a normalized signal-to-noise ratio for $k$ clusters as $r_1 = \frac{\Delta}{\sigma}\sqrt{\frac{\gamma_1}{1+\frac{kd}{(1-\alpha)m}}}$. Suppose after $S$ iterations of clustering algorithm in stage II, we have $k$ clusters $S_1,\dots, S_k$ of ERMs. We denote $\beta_j$ as the fraction of trimmed points and $\alpha_j$ as the fraction of adversarial points in cluster $S_j$. Then the maximum mis-clustering fraction is bounded by $\rho = \Gamma'(\frac{c}{r_1^2}+\sqrt{\frac{5k\log((1-\alpha)m)}{\gamma_1^2(1-\alpha)m}})+\max_{j\in [k]}\frac{\beta_j}{1-\alpha_j}$ for some constant $c$ where $\Gamma' = \max_{j\in [k]}\frac{1-\beta_j}{1-\alpha_j}$. Denote $C_1 = \frac{\Gamma'c\sigma^2}{\Delta^2}, C_2=\frac{\Gamma'c\sigma^2 }{\Delta^2 m}$ and $C_3 = \max_{j\in [k]}\frac{\beta_j}{1-\alpha_j}$. Plug $\rho$ into $\tilde{\alpha}_j$ and we get 
\begin{align}
	\tilde{\alpha}_j&\nonumber=\frac{\rho p_j+\alpha}{p_j+\alpha}\le\frac{\rho p_j+\alpha}{p+\alpha}\\
	&\nonumber=\frac{(C_1\frac{1}{\gamma_1}+C_2\frac{kd}{(1-\alpha)\gamma_1}+\frac{\Gamma'}{\gamma_1}\sqrt{\frac{5k\log((1-\alpha)m)}{(1-\alpha)m}}+C_3)p_j}{p+\alpha}+\frac{\alpha}{p+\alpha}\\
	&=\frac{(\frac{C_1(1-\alpha)}{p}+\frac{C_2kd}{p}+\frac{\Gamma'}{p}\sqrt{\frac{5k(1-\alpha)\log((1-\alpha)m)}{m}}+C_3)p_j}{p+\alpha}+\frac{\alpha}{p+\alpha} \label{equal.1}.
\end{align}
We first consider a special case that all clusters have the same number of clients. In this case, $p = p_j = \frac{1}{k}\forall j\in [k]$. Then (\ref{equal.1}) becomes 
\begin{align*}
	\tilde{\alpha}_j&=\frac{C_1(1-\alpha)}{\frac{1}{k}+\alpha}+\frac{C_2d}{\frac{1}{k}(\frac{1}{k}+\alpha)}+\frac{\Gamma'\sqrt{\frac{5(1-\alpha)\log((1-\alpha)m)}{m}}}{\frac{1}{\sqrt{k}}(\frac{1}{k}+\alpha)}+\frac{C_3\frac{1}{k}}{\frac{1}{k}+\alpha}+\frac{\alpha}{\frac{1}{k}+\alpha}.
\end{align*}
Then $\tilde{\alpha}_j = \mathcal{O}(\frac{d}{p^2}+\frac{\alpha}{p})$. This leads the error rate to have an extra dependence on $\frac{d^2}{p^2\sqrt{n}}$. \\
\\
In general case, $p_j <1\,\forall j\in [k]$. Then 
\begin{align*}
	\tilde{\alpha}_j&\le \frac{C_1(1-\alpha)}{p(p+\alpha)}+\frac{C_2kd}{p(p+\alpha)}+\frac{\Gamma'\sqrt{\frac{5k(1-\alpha)\log((1-\alpha)m)}{m}}}{p(p+\alpha)}+\frac{C_3}{p+\alpha}+\frac{\alpha}{p+\alpha}\\
	&=\mathcal{O}(\frac{d}{p^2}+\frac{\alpha}{p})
\end{align*}
which still gives the error rate an extra dependence on $\frac{d^2}{p^2\sqrt{n}}$.

\end{document}